\documentclass[conference]{IEEEtran}
\IEEEoverridecommandlockouts
\usepackage{cite}
\usepackage{amsmath,amssymb,amsfonts}
\usepackage{graphicx}
\usepackage{textcomp}
\usepackage{xcolor}
\usepackage{algorithm}
\usepackage{algpseudocode}

\newtheorem{theorem}{Theorem}

\newtheorem{proposition}{Proposition}

\newtheorem{proof}{Proof}

\def\BibTeX{{\rm B\kern-.05em{\sc i\kern-.025em b}\kern-.08em
    T\kern-.1667em\lower.7ex\hbox{E}\kern-.125emX}}
\begin{document}

\title{Post-Processing Mask-Based Table Segmentation for Structural Coordinate Extraction
}

\author{
\IEEEauthorblockN{Suren Bandara}
\IEEEauthorblockA{
\textit{University of Moratuwa, Sri Lanka} \\
surenbandara7@gmail.com
}
}

\maketitle

\begin{abstract}
Structured data extraction from tables plays a crucial role in document image analysis for scanned documents and digital archives. Although many methods have been proposed to detect table structures and extract cell contents, accurately identifying table segment boundaries (rows and columns) remains challenging, particularly in low-resolution or noisy images. In many real-world scenarios, table data are incomplete or degraded, limiting the adaptability of transformer-based methods to noisy inputs. Mask-based edge detection techniques have shown greater robustness under such conditions, as their sensitivity can be adjusted through threshold tuning; however, existing approaches typically apply masks directly to images, leading to noise sensitivity, resolution loss, or high computational cost. This paper proposes a novel multi-scale signal-processing method for detecting table edges from table masks. Row and column transitions are modeled as one-dimensional signals and processed using Gaussian convolution with progressively increasing variances, followed by statistical thresholding to suppress noise while preserving stable structural edges. Detected signal peaks are mapped back to image coordinates to obtain accurate segment boundaries. Experimental results show that applying the proposed approach to column edge detection improves Cell-Aware Segmentation Accuracy (CASA) a layout-aware metric evaluating both textual correctness and correct cell placement from 67\% to 76\% on the PubLayNet-1M benchmark when using TableNet~\cite{paliwal2019tablenet} with PyTesseract OCR~\cite{saoji2021text}. The method is robust to resolution variations through zero-padding and scaling strategies and produces optimized structured tabular outputs suitable for downstream analysis.
\end{abstract}
\begin{IEEEkeywords}
Table Structure Extraction, Mask-Based Segmentation, Column and Row Detection, Iterative Threshold–Convolution, Document Image Analysis, OCR Enhancement
\end{IEEEkeywords}

\section{Introduction}
Tables are a fundamental structure in scientific, financial, and administrative documents, providing structured representations of information that are critical for downstream processing and analysis. Automatic extraction of table structures enables efficient digitization and reuse of data from scanned documents and digital archives. However, accurately identifying table segment boundaries specifically, rows and columns remains challenging in real-world scenarios due to noise, low resolution, and degraded document quality.

Recent research in table structure recognition has predominantly focused on deep learning–based approaches, including convolutional and transformer-based models such as TableNet ~\cite{paliwal2019tablenet} and CascadeTabNet~\cite{prasad2020cascadetabnet}. These methods are effective at producing table masks that indicate the presence of rows and columns and are widely used as a first stage in table extraction pipelines. However, their outputs are often imperfect, particularly for noisy or low-resolution inputs, resulting in fragmented or misaligned masks. Moreover, these models are not easily adjustable once trained, limiting their adaptability to varying noise conditions encountered in practical applications.

Mask-based representations provide an interpretable intermediate form that allows post-processing without retraining models. Existing segmentation methods, however, typically apply masks directly to images to suppress non-table regions, which can amplify noise, degrade resolution, and lead to inaccurate boundary localization. More importantly, such approaches often prioritize structural separation without explicitly considering information preservation within table cells.

In this work, we argue that preserving textual information together with its correct structural context is essential for downstream processing. When table content is extracted without accurate row and column alignment, subsequent processing modules such as large language models (LLMs) or information extraction systems receive disordered or context-free text, leading to suboptimal or incorrect results. Therefore, the primary objective of this study is to balance structure preservation with content fidelity, ensuring that extracted words remain associated with their correct table cells.

To achieve this, we propose a lightweight, training-free signal-processing framework that operates on table masks generated by existing models such as TableNet ~\cite{paliwal2019tablenet} and CascadeTabNet~\cite{prasad2020cascadetabnet}. Instead of directly modifying image intensities, we model row and column transitions in the masks as one-dimensional signals. These signals are processed using multi-scale Gaussian convolution with progressively increasing variances, followed by statistical thresholding to suppress noise while emphasizing stable structural edges. The detected signal peaks are mapped back to the original image coordinates to obtain accurate segment boundaries.

To evaluate both structural correctness and information preservation, we introduce Cell-Aware Segmentation Accuracy (CASA), a layout-aware metric that considers a word to be correct only if it is both accurately recognized and placed within the correct table cell. Experimental results on the PubLayNet-1M benchmark demonstrate that applying the proposed method to column edge detection improves CASA from 67\% to 76\% when using TableNet-generated masks and PyTesseract OCR~\cite{saoji2021text}. These results indicate that the proposed approach enhances structural alignment without sacrificing textual content, producing optimized structured tabular outputs suitable for downstream analysis.
\section{Related Work}

Recent advances in table structure recognition have increasingly relied on mask-based deep learning architectures to localize and segment tabular components from document images. These approaches typically learn pixel-wise representations of tables, rows, columns, or cells, which are then used to guide downstream structure reconstruction and content extraction.

One of the early and influential works in this direction is TableNet~\cite{paliwal2019tablenet}, which introduced an end-to-end multi-task framework for jointly performing table detection and column segmentation. The model employs a shared VGG-19 encoder with two task-specific decoder branches, producing a table mask and a column mask aligned with the input image. These masks are primarily used to highlight column regions, after which heuristic rules are applied to infer row boundaries. Optical character recognition (OCR) is then performed within the masked regions to recover tabular content. By sharing representations between detection and structure-related tasks, TableNet~\cite{paliwal2019tablenet} demonstrated improved performance over single-task baselines.

Building on this idea, CascadeTabNet~\cite{prasad2020cascadetabnet} reformulated table detection and structure recognition as an instance segmentation problem. The approach leverages a Cascade Mask R-CNN architecture with an HRNet backbone to simultaneously predict table regions and individual table cell masks. Additionally, tables are classified as bordered or borderless, with different post-processing strategies applied accordingly. For borderless tables, cell-level masks are used to infer structure, whereas bordered tables rely on traditional line detection and text alignment techniques. Through iterative transfer learning and extensive data augmentation, CascadeTabNet~\cite{prasad2020cascadetabnet} achieves strong generalization across diverse datasets and table layouts.

A related line of work utilizes VGG-19-based convolutional architectures for table extraction and structure recognition~\cite{iqbal2025table}. These methods generate row and column masks to suppress non-tabular regions in the document image. The masked images are subsequently passed to OCR engines, and heuristic rules are used to organize the recognized text into a structured table format. By combining deep visual features with semantic cues, such approaches improve robustness in complex and noisy document layouts.

Despite their effectiveness, these mask-based methods share a fundamental limitation. In most cases, the predicted segmentation masks are treated primarily as visual filters, rather than as structured signals from which explicit geometric coordinates of rows and columns are directly derived. The conversion of column or cell masks into precise structural boundaries is often heuristic, implicit, or dataset-specific, and is not explicitly optimized during training. Moreover, applying segmentation masks directly to the input image can suppress fine-grained visual details near structural boundaries, leading to the loss of small but critical elements such as numeric characters or punctuation marks. This issue becomes particularly pronounced in densely populated or low-resolution tables, where OCR accuracy can degrade significantly (see Figure~\ref{fig:noisy-mask}).
\begin{figure}[h]
    \centering
    \includegraphics[width=0.8\linewidth]{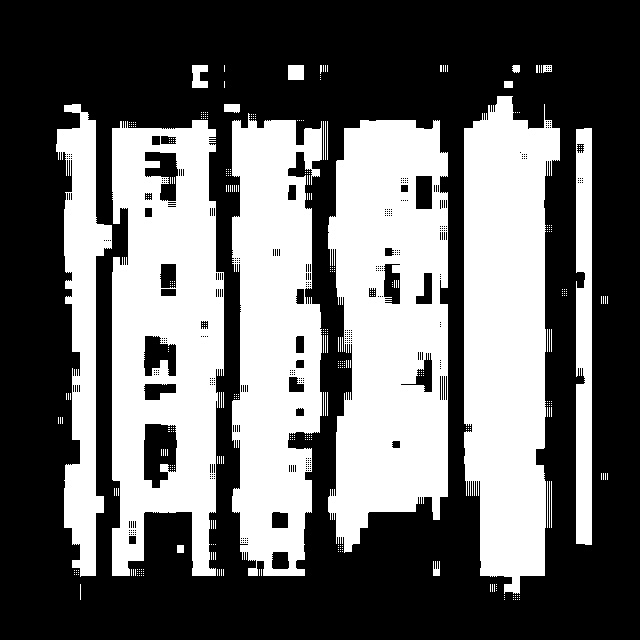}
    \caption{Example of a noisy column mask output illustrating the challenge of using mask-based segmentation directly for structural coordinate extraction.}
    \label{fig:noisy-mask}
\end{figure}

In contrast to these approaches, our method focuses on extracting structural information directly from mask distributions using a multi-scale convolutional filtering strategy, without degrading the original image content. This design provides a principled and computationally efficient middle ground between heuristic-based classical methods and resource-intensive end-to-end deep learning solutions.

\section{Methodology}

\subsection{Overview}

We propose a post-processing method that derives explicit structural coordinates directly from segmentation mask distributions, without modifying or masking the original document image. Instead of treating the mask as a visual filter, the proposed approach interprets it as a spatial signal encoding structural transitions.

The method aggregates pixel-wise transitions along one axis of the mask to identify stable boundary locations along the orthogonal axis. By accumulating and smoothing these transitions across multiple scales, consistent row or column boundaries emerge even when the predicted mask is noisy, incomplete, or fragmented. This enables reliable localization of structural regions while preserving all visual details in the original image for subsequent OCR and content extraction.

\subsection{Derivation of Raw Midpoint Accumulation}

Consider a binary mask $M \in \{0,1\}^{H \times W}$ generated by a segmentation model.  
To extract one-dimensional structural evidence along the $x$-direction, we aggregate information by scanning along the orthogonal direction $y$ (e.g., scanning vertically to localize horizontal separators, or horizontally to localize vertical separators).

For each fixed $y \in \{1,\dots,H\}$, define the one-dimensional binary signal
\begin{equation}
m_y(x) = M(y,x), \quad x \in \{1,\dots,W\}.
\end{equation}

We identify transition indices where the mask value changes:
\begin{equation}
\Gamma_y = \{x \in \{1,\dots,W-1\} \mid m_y(x) \neq m_y(x+1)\}.
\end{equation}

Each pair of consecutive transition points $(x_{y,i}, x_{y,i+1}) \in \Gamma_y$ defines an interval along the $x$-axis.  
For each such interval, we compute the midpoint location
\begin{equation}
\tilde{x}_{y,i} = \frac{x_{y,i} + x_{y,i+1}}{2}.
\end{equation}

Collecting midpoint samples across all $y$ yields a multiset
\begin{equation}
\mathcal{X} = \left\{ \tilde{x}_{y,i} \;\middle|\; y = 1,\dots,H,\; i = 1,\dots,|\Gamma_y|-1 \right\}.
\end{equation}

From the extracted midpoint samples, we construct a non-normalized accumulation function by aggregating their occurrences along the $x$-axis.

Let $\mathcal{X} = \{\tilde{x}_{y,i}\}$ denote the set of midpoint locations obtained across all scan positions. The raw accumulation function $g(x)$ is defined as
\begin{equation}
g(x) = \sum_{y=1}^{H} \sum_{i} \mathbf{1}\!\left(x = \tilde{x}_{y,i}\right),
\end{equation}
where $\mathbf{1}(\cdot)$ denotes the indicator function.

In discrete implementations, $g(x)$ corresponds to a histogram over the $x$-axis, whose magnitude reflects the number of scan lines supporting a midpoint at location $x$. This function is not normalized and therefore does not represent a probability density.

The raw discrete accumulation $g[x]$ can be normalized to obtain an empirical density:

\begin{equation}
f_0[x] = \frac{g[x]}{\sum_{x} g[x]},
\end{equation}

which satisfies
\begin{equation}
f_0[x] \ge 0, \quad \sum_x f_0[x] = 1.
\end{equation}

To construct a continuous representation $f_0(x)$ from the discrete bins, one can perform linear interpolation or kernel smoothing:

\begin{equation}
f_0(x) \approx \sum_{x_i} f_0[x_i] \, K_\epsilon(x - x_i),
\end{equation}

where $K_\epsilon$ is a smoothing kernel (e.g., Gaussian) with small width $\epsilon > 0$.  
This yields a continuous function $f_0(x)$ suitable for subsequent convolution and threshold–convolution steps.

\subsection{Problem Setting and Notation}

Due to measurement noise, spectral leakage, and finite sampling effects, $f_0(x)$ is assumed to be a distorted realization of an underlying latent distribution $p(x)$
\begin{equation}
p(x) = \sum_{k=1}^{K} \pi_k \, \delta(x - \mu_k)
\end{equation}
where $\pi_k \ge 0$, $\sum_k \pi_k = 1$

which is well approximated by a finite Gaussian mixture in real world scenarios:
\begin{equation}
f_0(x) = \sum_{k=1}^K \pi_k \, \mathcal{N}(x \mid \mu_k, \sigma_k^2) + \varepsilon(x),
\end{equation}

where \( \varepsilon(x) \) denotes additive noise.

The objective of the proposed method is to recover a regularized estimate of the density 
\[
f_0(x) \approx \sum_{k=1}^K \pi_k \, \mathcal{N}(x \mid \mu_k, \sigma_k^2),
\] 
which allows the identification of the dominant spectral modes (peaks) while suppressing noise-induced artifacts represented by $\varepsilon(x)$.  

By regularizing the observed density, we aim to preserve the main structure of the signal (the Gaussian components) and remove spurious fluctuations caused by additive noise.

\subsection{Variational Objective}

We define a latent energy functional:

\begin{equation}
\mathcal{E}[f] = D_{\mathrm{KL}}(f \,\|\, f_0) + \lambda \int f(x) \log f(x)\,dx + \mu \int \Phi_\theta(f(x))\,dx,
\end{equation}

where:
\begin{itemize}
    \item \(D_{\mathrm{KL}}(f \,\|\, f_0) = \int f(x) \log \frac{f(x)}{f_0(x)} dx \) is the data fidelity term,
    \item \( \int f \log f \, dx \) is an entropy-based smoothness regularization,
    \item \( \Phi_\theta(f) = \mathbf{1}[f < \theta] \) penalizes low-amplitude noise, and
    \item \(\lambda, \mu > 0\) control the balance between smoothing and noise suppression.
\end{itemize}

The optimal distribution minimizes this energy:

\begin{equation}
f^* = \arg\min_f \mathcal{E}[f].
\end{equation}

\subsection{Proposed Iterative Threshold–Convolution Scheme}

We approximate the solution using alternating minimization, which corresponds to our practical algorithm:

\begin{enumerate}
    \item \textbf{Thresholding step (noise suppression)}:
    \begin{equation}
    f^{(n+\frac12)} = \arg\min_f \, \mu \int \Phi_\theta(f) + \frac{1}{2} \| f - f^{(n)} \|_2^2,
    \end{equation}
    implemented as:
    \[
    f^{(n+\frac12)}(x) =
    \begin{cases}
    f^{(n)}(x), & f^{(n)}(x) \ge \theta \\
    0, & f^{(n)}(x) < \theta
    \end{cases}
    \]

    Given a threshold $\theta > 0$, define the nonlinear truncation operator
    \begin{equation}
    (\mathcal{T}_\theta f)(x)
    =
    f(x)\,\mathbf{1}\{f(x) \ge \theta\}.
    \end{equation}

    This operation suppresses low-amplitude frequency components typically associated with noise floor effects.

    \item \textbf{Gaussian convolution step (smoothness)}:
    \begin{equation}
    f^{(n+1)} = \arg\min_f \, \lambda \int f \log f + \frac{1}{2} \| f - f^{(n+\frac12)} \|_2^2,
    \end{equation}
    whose solution is convolution with a zero-mean Gaussian kernel \( G_{\sigma}(x) \):
    \[
    f^{(n+1)} = f^{(n+\frac12)} * G_\sigma.
    \]
    Here
    \begin{equation}
    G_\sigma(x) = \frac{1}{\sqrt{2\pi\sigma^2}}
    \exp\left(-\frac{x^2}{2\sigma^2}\right)
    \end{equation}
\end{enumerate}

The proposed algorithm alternates thresholding and convolution:
\begin{equation}
\boxed{
f_{n+1}
=
\mathcal{G}_\sigma \big( \mathcal{T}_\theta(f_n) \big),
\quad n \ge 0.
}
\end{equation}

Each iterate is renormalized to ensure unit mass.
In the process we should be able to reduce the $D_{\mathrm{KL}}(f \,\|\, f_0)$ between the estimated distribution and the latent distribution.
This can be done by controlling iterations. 
 After applying the iterative threshold–convolution process, the resulting distribution exhibits well-separated bell-shaped peaks. The locations of these peaks correspond to the underlying structural coordinates (e.g., approximate column or row positions) in the input data.
 The pseudo-code of the method is summarized in Algorithm ~\ref{alg:threshold_convolution}.

\begin{algorithm}[t!]
\caption{Iterative Threshold--Convolution Regularization from Mask}
\label{alg:threshold_convolution}
\begin{algorithmic}[1]
\State \textbf{Input:} Binary mask $M \in \{0,1\}^{H \times W}$

\State \textbf{Midpoint extraction:}
\For{$y = 1$ to $H$}
    \For{each contiguous interval $i$ in row $y$ where $M(y,\cdot)$ changes value}
        \State Detect index pair $(x^{(l)}_{y,i}, x^{(r)}_{y,i})$ marking the interval
        \State Compute midpoint:
        \[
        \tilde{x}_{y,i} \gets \frac{x^{(l)}_{y,i} + x^{(r)}_{y,i}}{2}
        \]
        \State Append $\tilde{x}_{y,i}$ to midpoint list
    \EndFor
\EndFor

\State \textbf{Discrete accumulation:}
\For{each midpoint $\tilde{x}_{y,i}$ in the list}
    \State Increment histogram bin:
    \[
    g[\lfloor \tilde{x}_{y,i} \rfloor] \gets g[\lfloor \tilde{x}_{y,i} \rfloor] + 1
    \]
\EndFor

\State \textbf{Normalization:}
\[
f_0[x] \gets \frac{g[x]}{\sum_x g[x]}, \quad \sum_x f_0[x] = 1
\]

\State \textbf{Continuous representation (smoothing):}
\[
f_0(x) \approx \sum_{x_i} f_0[x_i] \, K_\epsilon(x - x_i),
\]
where $K_\epsilon(x)$ is a normalized kernel (e.g., Gaussian) with
\(\int K_\epsilon(x)\,dx = 1\), ensuring
\(\int f_0(x)\,dx = 1\).

\State Estimate initial density $f_0(x)$ from frequency data
\For{$n = 0$ to $N-1$}
    \State $f_n^\theta(x) \gets f_n(x)\mathbf{1}[f_n(x) \ge \theta_n]$
    \State $f_{n+1}(x) \gets f_n^\theta * \mathcal{N}(0,\sigma_n^2)$
    \State Normalize $f_{n+1}$
\EndFor

\State Identify local maxima in $f_N(x)$:
\[
\{x_k\} = \{ x \mid f_N(x-1) < f_N(x) > f_N(x+1) \}
\]
\State \textbf{Output :  $\{x_k\}$}
\end{algorithmic}
\end{algorithm}

\subsection{Parameter Selection}

\begin{itemize}
    \item \textbf{Threshold \(\theta\)}:  
    \(\theta \approx (2 \sim 4) \times \hat{\sigma}_\varepsilon\), where \(\hat{\sigma}_\varepsilon\) is an estimate of the noise level. Alternatively, remove the bottom 5–15\% of amplitudes.
    \item \textbf{Gaussian smoothing \(\sigma\)}:  
    Choose \(\sigma\) smaller than the typical Gaussian component width. Multi-scale smoothing (\(\sigma_n\) increasing per iteration) can improve convergence.
    \item \textbf{Number of iterations \(n\)}:  
    Typically 2–5 iterations is sufficient. Fewer leaves noise; more over-smooths and merges peaks.
\end{itemize}

\section{Theoretical Analysis}

\subsection{Effect of Gaussian Convolution}

\begin{proposition}[Smoothing as Low-Pass Filtering]
Let $f_0 \in L^1(\mathbb{R})$ have finite second moment.  
Convolution of $f_0$ with a Gaussian kernel $\mathcal{G}_\sigma$,
\begin{equation}
f_1 = f_0 * \mathcal{G}_\sigma,
\end{equation}
produces a smoother distribution with the same mean 
\(\mu = \int x f_0(x)\,dx\) and variance
\begin{equation}
\text{Var}(f_1) = \text{Var}(f_0) + \sigma^2.
\end{equation}
\end{proposition}

\begin{proof}
In the Fourier domain, convolution corresponds to multiplication:
\[
\widehat{f_1}(\omega) = \hat{f}_0(\omega) \, \exp\Big(-\frac{1}{2}\sigma^2 \omega^2\Big).
\]

The factor $\exp(-\frac{1}{2}\sigma^2 \omega^2)$ attenuates high-frequency components (large $\omega$).  
Low-frequency components (small $\omega$) are largely preserved.  
Therefore, Gaussian convolution acts as a low-pass filter, smoothing the distribution while preserving its main structure.
\end{proof}

\subsection{Preservation and Variance Increase of Gaussian Mixtures}

\begin{theorem}[Closure under Gaussian Convolution]
\label{thm:mixture}
Let
\[
f(x) = \sum_{k=1}^K \pi_k \, \mathcal{N}(\mu_k,\sigma_k^2)
\]
be a finite Gaussian mixture.  
Convolution with a Gaussian kernel $\mathcal{G}_\sigma$ yields
\[
f * \mathcal{G}_\sigma = \sum_{k=1}^K \pi_k \, \mathcal{N}(\mu_k,\sigma_k^2 + \sigma^2).
\]
\end{theorem}

\begin{proof}
Convolution of a Gaussian with another Gaussian produces a Gaussian with mean preserved and variance increased by the kernel variance.  
By linearity, the mixture weights \(\pi_k\) are unchanged, so the mixture structure is preserved.
\end{proof}
Thus, Gaussian convolution polishes each component of the mixture by increasing its variance, smoothing the overall distribution while maintaining the original peak locations.

\subsection{Role of Thresholding in Mode Selection}

\begin{theorem}[Nonlinear Mode Suppression]
\label{thm:threshold}
The operator $\mathcal{T}_\theta$ eliminates low-amplitude components and reduces higher-order cumulants of $f(x)$.
\end{theorem}

\begin{proof}
Thresholding introduces a nonlinear projection that removes small-amplitude regions contributing disproportionately to kurtosis and skewness. Subsequent convolution redistributes mass smoothly, yielding distributions with reduced higher-order moments.
\end{proof}

This explains why repeated threshold–convolution cycles stabilize dominant Gaussian-shaped spectral modes.

\section{Experiments and Results}

\subsection{Experimental Setup}

This section describes how the proposed post-processing method was implemented and evaluated. The focus of the experiments is on extracting reliable column coordinates from predicted segmentation masks and assessing their effectiveness in preserving tabular content.

\subsubsection{Dataset and Mask Generation}

Experiments were conducted using the PubTable-1M dataset~\cite{smock2022pubtables}. XML annotations provided ground-truth cell coordinates, which were used to generate binary column segmentation masks. All images and masks were resized to $1024 \times 1024$, and pixel values were normalized to the range $[0,1]$.

\subsubsection{Column Mask Prediction}

Column masks were generated using a TableNet-based architecture ~\cite{paliwal2019tablenet}  with VGG-19  as backbone trained on full table images without cropping. The model predicts column masks as part of its multi-task learning framework using binary cross-entropy loss. Since the contribution of this work lies in post-processing rather than segmentation model design, training-related details are not further discussed.

\subsubsection{Post-Processing and Boundary Extraction}

The proposed post-processing method was applied exclusively to the predicted column masks. For each mask, pixel-wise transitions were aggregated along the horizontal axis to form a one-dimensional density signal. Column boundary candidates were extracted using an iterative threshold–convolution procedure.

Two iterations were applied sequentially. At each iteration, the signal was smoothed using a Gaussian kernel and thresholded relative to its current distribution. Local maxima of the final smoothed signal were selected as column boundary coordinates and rescaled to the original image dimensions.

\subsubsection{Row Differentiation via OCR}

Row boundaries were not derived from segmentation masks. Instead, extracted column regions were processed using PyTesseract~\cite{saoji2021text}. Vertical spacing between recognized text lines was analyzed to infer row transitions. This enables implicit row separation based on textual layout without introducing an additional row segmentation stage.

\subsubsection{Hyperparameter Tuning}

Hyperparameters for the iterative threshold–convolution procedure were manually tuned through empirical evaluation:

\begin{itemize}
    \item \textbf{Number of iterations:} Two sequential iterations applied.
    \item \textbf{Thresholding:} First iteration $1.5 \times \sigma$, second iteration $1.0 \times \sigma$, where $\sigma$ is the standard deviation of the current signal.
    \item \textbf{Gaussian smoothing parameters:}  First iteration: $(\sigma=5, \mu=0)$, second iteration: $(\sigma=7, \mu=0)$
\end{itemize}

\subsection{Results}

The proposed method was evaluated using the \textbf{Content-Aware Segmentation Accuracy (CASA)} metric introduced earlier. CASA measures the proportion of correctly identified words after segmentation and OCR, reflecting how well the extracted structural boundaries preserve tabular content.

{\footnotesize
\begin{equation}
\text{CASA (\%)} =
\frac{\text{Number of correctly identified words in segmented regions}}
     {\text{Total number of words in ground-truth cells}}
\times 100
\end{equation}
}

For comparison, applying OCR directly on the raw column mask predictions without post-processing yielded a CASA of \textbf{67\%}, which increased to \textbf{76\%} CASA after applying our method, demonstrating the effectiveness of the iterative smoothing and thresholding in recovering accurate table structures.

\begin{figure}[!t]
\centering
\includegraphics[width=0.45\textwidth]{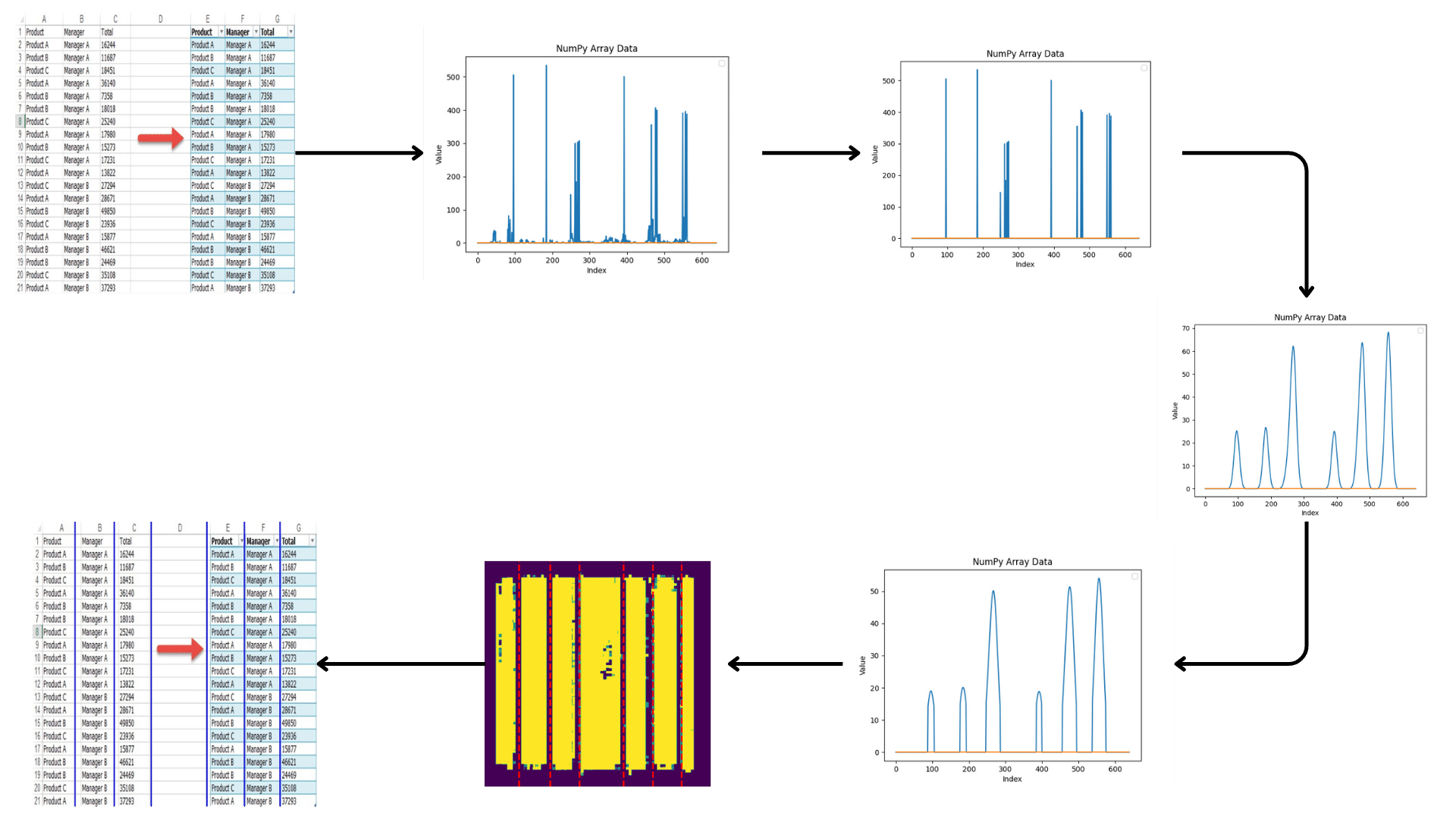}
\caption{Evolution of the density $f_n(x)$ during iterative threshold–convolution. Dominant peaks corresponding to column boundaries become increasingly pronounced across iterations.}
\label{fig:density_example}
\end{figure}

These results demonstrate that explicit column coordinates can be reliably extracted directly from noisy segmentation masks without modifying the original image. The proposed approach preserves textual content and is independent of the underlying segmentation network, making it applicable to a wide range of mask-based table structure recognition models.

\section{Limitations}

While the proposed post-processing framework effectively extracts column and row coordinates from mask-based table segmentation outputs, it is limited to segment-wise table structures. The method assumes that table layouts are relatively regular and that the mask captures contiguous rows and columns. In scenarios involving complex or highly nested tables, irregular spanning cells, or multi-level hierarchies, the accumulation-based approach may fail to correctly identify structural boundaries. Additionally, the method relies on the quality of the predicted masks; severely fragmented or low-confidence masks can reduce the accuracy of extracted coordinates. Future extensions may address these challenges by integrating hierarchical modeling or graph-based representations for more complex table structures.
\section{Conclusion}

This work presents a post-processing framework for extracting structural coordinates from mask-based table segmentation outputs. By aggregating midpoints from predicted masks and applying iterative thresholding combined with Gaussian smoothing, the method produces well-separated peaks corresponding to table columns and rows. This approach accurately identifies structural boundaries even in noisy or fragmented mask predictions, while preserving the original image for subsequent OCR.

Experiments with TableNet-generated masks demonstrate that our method enhances structural extraction accuracy without degrading visual content. The framework is generalizable and can be applied to other mask-based table recognition models, offering a practical enhancement to existing table detection pipelines. Future work will explore adaptive thresholding and multi-scale smoothing to further improve robustness for low-resolution or highly noisy tables.

\bibliographystyle{ieeetr} 
\bibliography{references}

@inproceedings{paliwal2019tablenet,
  title={Tablenet: Deep learning model for end-to-end table detection and tabular data extraction from scanned document images},
  author={Paliwal, Shubham Singh and Vishwanath, D and Rahul, Rohit and Sharma, Monika and Vig, Lovekesh},
  booktitle={2019 International Conference on Document Analysis and Recognition (ICDAR)},
  pages={128--133},
  year={2019},
  organization={IEEE}
}

@inproceedings{smock2022pubtables,
  title={PubTables-1M: Towards comprehensive table extraction from unstructured documents},
  author={Smock, Brandon and Pesala, Rohith and Abraham, Robin},
  booktitle={Proceedings of the IEEE/CVF Conference on Computer Vision and Pattern Recognition},
  pages={4634--4642},
  year={2022}
}

@article{saoji2021text,
  title={Text recognition and detection from images using pytesseract},
  author={Saoji, Saurabh and Eqbal, A and Vidyapeeth, B},
  journal={J Interdiscip Cycle Res},
  volume={13},
  pages={1674--1679},
  year={2021}
}

@inproceedings{prasad2020cascadetabnet,
  title={CascadeTabNet: An approach for end to end table detection and structure recognition from image-based documents},
  author={Prasad, Devashish and Gadpal, Ayan and Kapadni, Kshitij and Visave, Manish and Sultanpure, Kavita},
  booktitle={Proceedings of the IEEE/CVF conference on computer vision and pattern recognition workshops},
  pages={572--573},
  year={2020}
}

@article{iqbal2025table,
  title={Table Extraction with Table Data Using VGG-19 Deep Learning Model},
  author={Iqbal, Muhammad Zahid and Garg, Nitish and Ahmed, Saad Bin},
  journal={Sensors},
  volume={25},
  number={1},
  pages={203},
  year={2025},
  publisher={MDPI}
}

\end{document}